\documentclass[11pt]{article}

\usepackage[preprint]{acl}

\usepackage{times}
\usepackage{latexsym}

\usepackage[T1]{fontenc}

\usepackage[utf8]{inputenc}

\usepackage{microtype}

\IfFileExists{inconsolata.sty}{\usepackage{inconsolata}}{}

\usepackage{graphicx}

\usepackage{amsmath}
\usepackage{multirow}
\usepackage{makecell}
\usepackage{colortbl}
\usepackage{amssymb}
\usepackage{mathtools}
\usepackage{subcaption}
\IfFileExists{algorithm.sty}{\usepackage{algorithm}}{}
\IfFileExists{algorithmic.sty}{\usepackage{algorithmic}}{}
\usepackage{amsthm}
\IfFileExists{bbm.sty}{\usepackage{bbm}}{}
\usepackage{natbib}
\usepackage{threeparttable} 
\usepackage{arydshln}
\usepackage{wrapfig}

\usepackage[T1]{fontenc}    
\usepackage{hyperref}       
\usepackage{url}            
\usepackage{booktabs}       
\usepackage{amsfonts}       
\usepackage{nicefrac}       
\usepackage{xcolor}         

\usepackage{cleveref}

\title{Asynchronous On-Policy Self-Distillation}

\author{
 \textbf{Xidong Wu \textsuperscript{1}},
 \textbf{Qi Qi \textsuperscript{2}},
 \textbf{Yuqiong Ji\textsuperscript{3}},
 \textbf{Shangqian Gao\textsuperscript{4}}
\\
 \textsuperscript{1}University of Pittsburgh,
 \textsuperscript{2}University of Iowa,
 \textsuperscript{3}University of Pennsylvania,
 \textsuperscript{4}Florida State University
}

\begin{document}
\maketitle
\begin{abstract}
Unlike on-policy distillation (OPD), which depends on a separate teacher model, on-policy self-distillation (OPSD) leverages a single language model acting simultaneously as the student and the teacher, with the teacher conditioned on privileged information. Despite this architectural simplicity, synchronous rollout generation dominates overall training time, creating a severe systems bottleneck when scaling OPSD. To address this, we introduce Asynchronous On-Policy Self-Distillation (AsyncOPSD), a framework that decouples rollout generation from learner updates to improve training throughput. Since the teacher and student share parameters, AsyncOPSD eliminates the multi-model coordination overhead inherent to OPD architectures in asynchronous training. As the student policy evolves, buffered trajectories generated by older policies become stale, potentially degrading performance. To mitigate this, we introduce a distillation objective based on staleness-aware policy control and token-level entropy weighting. Across three model sizes and mathematical reasoning benchmarks, AsyncOPSD improves the training throughput by $1.96$--$2.25\times$ over synchronous OPSD while increasing average benchmark accuracy.
\end{abstract}

\section{Introduction}
On-policy distillation (OPD) has been used as an effective approach in the large language model (LLM) post-training \cite{yang2025qwen3,zeng2026glm,gu2024minillm}. A student is trained on its own generated on-policy responses using token-level feedback from a teacher. 
Training on student-generated prefixes mitigates the train-test distribution shift inherent in distilling from fixed trajectories \cite{agarwal2024policy}.
OPD typically uses a separate large teacher model while On-policy self-distillation (OPSD) instead instantiates both roles from a single language model. The privileged information, including a verified solution reasoning trace or feedback, augments the teacher’s context, while the student receives only the question. Matching their token-level distributions transfers the benefit of this additional context to the student without requiring an external teacher \cite{zhao2026self}.

Despite this simplification, large-scale OPSD and OPD inherit the on-policy rollout bottleneck. Long reasoning traces dominate training time, and strict synchronization forces the learner to wait for new rollouts, leaving compute resources idle. While asynchronous execution improves utilization by generating rollouts and updating the learner concurrently, buffered trajectories quickly become stale. Training on this outdated data introduces a distribution mismatch
and creates a performance–efficiency trade-off \cite{kang2026asyncopd}. Recent work explores this trade-off in asynchronous OPD \cite{wang2026divopd, rang2026near, kang2026asyncopd}. 
These studies highlight the importance of designing asynchronous pipelines. 
We study asynchronous execution for OPSD by overlapping rollout generation and learner-side privileged-context distillation under a fixed GPU budget.
\begin{figure*}[th]
  \centering
  \includegraphics[width=\textwidth]{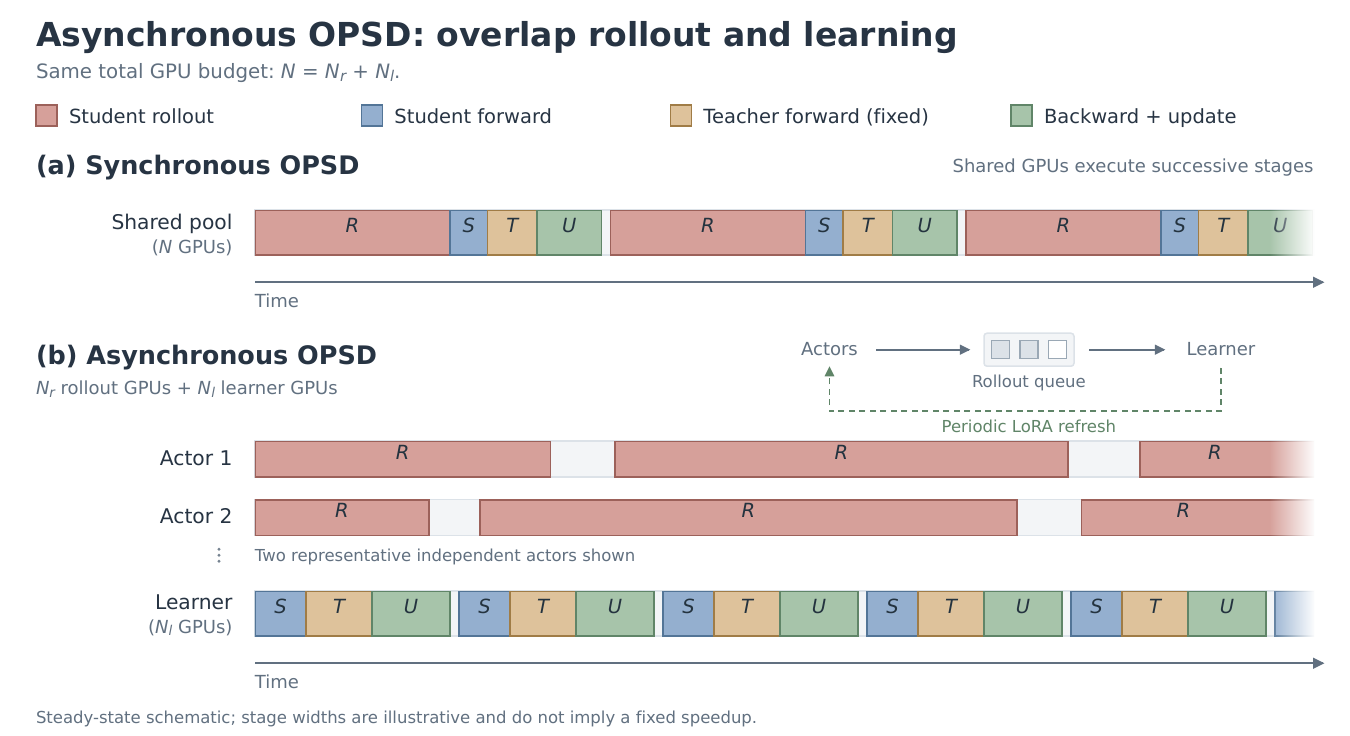}
  \caption{\textbf{Overview of AsyncOPSD.}
  Synchronous and asynchronous OPSD under the same budget.
  (a) Synchronous OPSD executes rollout and learning sequentially on $N$ GPUs. (b) Asynchronous OPSD assigns $N_r$ GPUs to independent rollout actors and $N_l$ GPUs
  to a single distributed learner ($N = N_r + N_l$). Actors add rollout microbatches in a shared buffer, and the learner periodically updates policy to the actors. 
  }
  \label{fig:async-opsd-overview}
\end{figure*}

We introduce Asynchronous On-Policy Self-Distillation (AsyncOPSD), a framework that decouples rollout generation from learner updates. Learners update model weights asynchronously from a shared replay buffer and periodically export the current policy (model weights), and distributed rollout actors generate trajectories after synchronizing the refreshed policy. At each update, learners evaluate buffered trajectories under both the student and a privileged teacher context. Computing teacher supervision on the learner side eliminates teacher-scoring stages and the communication overhead of transmitting cached teacher distributions, substantially simplifying system coordination compared to AsyncOPD \cite{kang2026asyncopd}.


However, buffered sequences inevitably reflect the behavioral distribution of older policies. Our theoretical analysis bounds the discrepancy using prefix-distribution mismatch and an entropy-dependent loss bound. This motivates two complementary controls: 1) Staleness-Aware Policy Control: Limits the age of rollouts accepted for training; 2) Token-Level Entropy Weighting: Reduces the relative contribution of high-entropy teacher predictions. These techniques act on rollout selection and token-level loss weighting, respectively. Systematic ablation studies evaluate their efficacy. Our main contributions are summarized as follows:

\begin{itemize}
\item  AsyncOPSD Framework: Introduced an asynchronous on-policy self-distillation architecture that overlaps rollout generation with optimization while computing privileged teacher supervision directly within the learner, removing the multi-model coordination overhead in asynchronous OPD training.
\item Staleness \& Uncertainty Mitigation:  Analyzed how prefix staleness and teacher entropy enter an objective-discrepancy bound and proposed staleness-aware policy control and token-level entropy weighting to stabilize training.
\item Empirical Acceleration: Evaluated AsyncOPSD across challenging reasoning benchmarks, achieving $1.96$--$2.25\times$ the sample throughput of synchronous OPSD and increasing average benchmark accuracy 
by $0.37$--$1.10$ percentage points across three model sizes.
\end{itemize}

\section{Related Work}
\subsection{On-Policy and Self-Distillation}
On-policy distillation (OPD) reduces the mismatch between training and inference by supervising trajectories sampled from the student \cite{agarwal2024policy, gu2023minillm, lu2025onpolicy, ko2024distillm,yang2026learning}. Per-token KL-divergence feedback from a teacher model in the OPD also offers a dense learning signal at every step to reinforce sub-optimal reasoning. Recent studies optimize transfer efficiency via thinking-pattern compatibility \cite{li2026rethinking}, teacher-entropy regularization \cite{jin2026entropy}, and uncertainty-disagreement token filtering \cite{xu2026tip}. Other OPD efficiency methods adapt rollout horizons~\cite{tan2026adaptive} or combine rollout-depth control with turn-normalized supervision~\cite{zhou2026turnopd}.
Self-distillation extends self-taught reasoning paradigms \cite{zelikman2022star} by removing external teachers entirely, extracting supervision from self-generated reasoning traces \cite{zhao2026self}, environment execution feedback \cite{hubotter2026reinforcement}, hybrid reward-distillation objectives on failed rollouts \cite{li2026unifying}, or majority-vote pseudo-solutions~\cite{li2026unsupervised}. POPO studies positive-only policy optimization with implicit negative gradients~\cite{xu2026positive}.
Recent work examines reference-induced shortcuts~\cite{shen2026purified}, reasoning degradation~\cite{kaur2026rethinking}, reduced output diversity~\cite{nicolicioiu2026diversity}, and sensitivity to reasoning modes and teacher prefixes~\cite{li2026diagnosing} in OPSD. Broader OP(S)D diagnostics identify supervision mismatch and objective-dependent instability~\cite{zhu2026many}.

\subsection{Asynchronous Reinforcement Learning} 
Decoupling rollout generation from policy optimization provides the foundation for scalable distributed RL \cite{mnih2016asynchronous, espeholt2018impala}. In LLM post-training, asynchronous architectures maximize accelerator utilization through pipeline overlapping \cite{noukhovitch2025faster}, worker load balancing \cite{fu2026areal}, disaggregated streaming generation \cite{zhong2025streamrl}, and dynamic trajectory repacking \cite{sheng2026laminar}. To mitigate distribution mismatch caused by policy lag, algorithmic solutions employ proximal interpolation \cite{li20253po} and second-moment importance-weight bounding \cite{zheng2026prosperity}. However, while RL relies on scalar feedback, distillation depends on dense prefix-conditioned teacher distributions, creating distinct staleness dynamics.

\subsection{Asynchronous On-Policy Distillation} 
Asynchronous OPD introduces coupled systems and distribution bottlenecks when integrating teacher evaluation into concurrent pipelines. Systems frameworks address these bottlenecks via pipelined scheduling \cite{sheng2025hybridflow} and distributed hidden-state streaming to reduce vocabulary communication \cite{zhang2026kdflow}. Algorithmically, AsyncOPD \cite{kang2026asyncopd} analyzes staleness under cached teacher scores, demonstrating divergence sensitivities and the inherent limits of action reweighting. Extending this, f-OPD \cite{chen2026boldsymbol} separates policy drift from supervision drift using freshness weighting and adaptive refreshes and DivOPD \cite{wang2026divopd} prioritizes high-disagreement turns across stale queues. 

However, asynchronous on-policy self-distillation is still underexplored. To address this gap, we introduce AsyncOPSD, which overlaps rollout generation and learner updates under a fixed GPU budget and simplifies distributed coordination by recomputing privileged teacher supervision directly within the learner.

\section{Preliminaries}

For next-token distributions $p$ and $q$ over vocabulary $\mathcal V$, define $D_{\mathrm{KL}}(p\|q)=\sum_{v\in\mathcal V}p(v)\log\frac{p(v)}{q(v)}$. 

\paragraph{On-Policy Distillation.}
On-policy distillation (OPD)~\cite{li2026rethinking} trains a student $p_S$ with parameters $\theta$ using a separate, fixed teacher $p_T$. Given a prompt $x\sim\mathcal{S}_x$, where $\mathcal{S}_x$ is the prompt dataset, the student samples $\hat y\sim p_S(\cdot\mid x)$. Both models evaluate the same generated prefixes, yielding $p_{S,n}=p_S(\cdot\mid x,\hat y_{<n})$ and $p_{T,n}=p_T(\cdot\mid x,\hat y_{<n})$. The objective is
\begin{equation}
\mathcal{L}_{\mathrm{OPD}}(\theta)
=\mathbb{E}_{\substack{x\sim\mathcal{S}_x\\\hat y\sim p_S(\cdot\mid x)}}\!\left[
\frac{1}{|\hat y|}\sum_{n=1}^{|\hat y|}
D(p_{T,n}\|p_{S,n})\right]\nonumber
\end{equation}
commonly instantiated with reverse KL. The teacher remains fixed during training.

\paragraph{On-Policy Self-Distillation.}
On-policy self-distillation (OPSD)~\cite{zhao2026self} uses a single LLM as both teacher and student by conditioning on asymmetric prompt contexts: \textbf{Student Policy $p_S(\cdot \mid x)$:} Observes only the original problem prompt $x$, faithfully replicating inference conditions. \textbf{Teacher Policy $p_T(\cdot \mid x, y^\star)$:} Receives privileged context consisting of the question $x$ paired with verified reference reasoning, ground-truth answers $y^\star$ or any other useful feedback.

Given problem-solution pairs $(x, y^\star)$ from a reasoning dataset $\mathcal{S}$, the student samples $\hat{y} \sim p_S(\cdot \mid x)$. Using the divergence $D$ and setting $p_{S,n}=p_S(\cdot\mid x,\hat y_{<n})$ and $p_{T,n}=p_T(\cdot\mid x,y^\star,\hat y_{<n})$, the objective becomes
\begin{equation}
\mathcal{L}_{\mathrm{OPSD}}(\theta)
=\mathbb{E}_{\substack{(x,y^\star)\sim\mathcal S\\\hat y\sim p_S(\cdot\mid x)}}\!\left[
\frac{1}{|\hat y|}\sum_{n=1}^{|\hat y|}
D(p_{T,n}\|p_{S,n})\right] \nonumber
\end{equation}
Gradients are backpropagated solely through the student policy $p_S$, while the privileged teacher $p_T$ acts as a stop-gradient target.

\section{AsyncOPSD: Framework, Analysis, and Method}
\label{sec:asyncopsd_method}
\subsection{Asynchronous Training Framework}
Extending On-Policy Self-Distillation (OPSD), we implement a baseline asynchronous variant (AsyncOPSD) by decoupling trajectory rollout generation from policy optimization across disjoint GPU pools. The GPU cluster is partitioned into the learner pool and the rollout actor pool. 
The learners execute forward passes, compute gradients, and optimize the model with the batch in the queue from rollout actors and periodically export the current policy (model weights).
Concurrently, in the rollout actor pool, distributed rollout actors host independent, single-GPU vLLM instances to sample after loading the refreshed policy, pushing generated trajectories into a shared rollout queue.  

Consistent with standard OPSD, no auxiliary teacher model is required; the student is optimized using a clipped forward-KL distillation objective. Model export frequency is governed by \texttt{sync\_step}, enabling generation and gradient computation to overlap seamlessly at the cost of bounded policy staleness. To prevent memory overload and control queue delays, the system limits queue size based on learner microbatches.

\subsection{Staleness--Entropy Analysis}
\label{sec:async_opsd_theory}

\ifcsname aopsdlemma\endcsname\else
  \newtheorem{aopsdlemma}{Lemma}
\fi
\ifcsname aopsdassumption\endcsname\else
  \newtheorem{aopsdassumption}{Assumption}
\fi
\ifcsname aopsdtheorem\endcsname\else
  \newtheorem{aopsdtheorem}{Theorem}
\fi

Asynchronous training potentially degrades performance. We provide an analysis relating this degradation to student rollout staleness and the privileged teacher entropy.
We first analyze the evolving shared-teacher formulation, in which privileged teacher and student predictions use the current parameters $\theta_k$. For the fair comparison in \Cref{sec:asyncopsd_setup}, we freeze the teacher policy to be the initial policy following the OPSD protocol~\cite{zhao2026self}. Appendix~\ref{app:aopsd_frozen_teacher} establishes the same form of staleness--entropy bound.

\paragraph{Shared-model formulation.}
Let $k$ denote the learner update and $r\le k$ the policy version used
for rollout generation. A training example consists of a problem $x$
and privileged information $a$, such as a verified answer.
For a response prefix $y_{<h}$, 
\begin{equation}
\begin{aligned}
s_h &= (x,a,h,y_{<h}),\\
p_k(\cdot\mid s_h)
  &= \pi_{\theta_k}(\cdot\mid x,y_{<h}),\\
q_k(\cdot\mid s_h)
  &= \pi_{\theta_k}(\cdot\mid x,a,y_{<h}).
\end{aligned} \nonumber
\end{equation}
Teacher and student share the same model. The teacher output is detached during optimization. The per-prefix forward-KL loss is
\begin{equation}
F_k(s)
=
\operatorname{KL}\!\left(
\operatorname{sg}[q_k(\cdot\mid s)]
\,\middle\|\,
p_k(\cdot\mid s)
\right) \nonumber
\label{eq:aopsd_token_loss}
\end{equation}
where $\operatorname{sg}$ denotes stop-gradient. At a fixed $k$,
stop-gradient affects differentiation but not the numerical value of
$F_k$.
Let $\mu_r$ and $\mu_k$ be prefix distributions induced by the rollout
and current policies, respectively. Both use the same distribution of
$(x,a)$ and the same fixed distribution over token positions
$h\in\{1,\ldots,T\}$. A finite horizon and a common continuation
convention after termination make these distributions well-defined. We consider a fixed rollout version $r$.
The asynchronous and synchronous population objectives are
\begin{equation}
\begin{aligned}
J_{\mathrm{async}}^k &= \mathbb{E}_{s\sim\mu_r}[F_k(s)],\\
J_{\mathrm{sync}}^k  &= \mathbb{E}_{s\sim\mu_k}[F_k(s)],\\
\Delta_k &= \left|J_{\mathrm{async}}^k-J_{\mathrm{sync}}^k\right|.
\end{aligned}
\label{eq:aopsd_objectives}
\end{equation}
This formulation follows the objective-discrepancy perspective of $f$-OPD~\cite{chen2026boldsymbol}, while specializing the supervision protocol. Both learner-side distributions are recomputed at $\theta_k$ on the stored prefix. 


The following theorem bounds the discrepancy between asynchronous and synchronous training objectives in terms of prefix staleness and teacher entropy, under the assumptions in Appendix~\ref{app:aopsd_theory}.

\begin{aopsdtheorem}[Joint staleness--entropy bound]
\label{thm:aopsd_discrepancy}
Fix a learner update $k$ and rollout version $r\le k$, with the prefix
laws and objectives in Eq.~\eqref{eq:aopsd_objectives}. Suppose
Assumptions~\ref{ass:aopsd_sampling} and~\ref{ass:aopsd_contrast}
hold, with $B\ge0$ the uniform privileged logit contrast bound in
Assumption~\ref{ass:aopsd_contrast}.
Define the prefix likelihood ratio, its staleness measure, and the
current teacher entropy by
\begin{equation}
\begin{aligned}
\rho_{k,r}(s_h)
&:=\frac{\mu_k(s_h)}{\mu_r(s_h)}
=\prod_{u<h}\frac{p_k(y_u\mid s_u)}{p_r(y_u\mid s_u)},\\
d_{k,r}(s)&:=|1-\rho_{k,r}(s)|,\\
H_k(s)&:=-\sum_{v\in\mathcal V}q_k(v\mid s)\log q_k(v\mid s),\\
C_B&:=\frac{B(e^B-1)}{2}.
\end{aligned}
\label{eq:aopsd_entropy}
\end{equation}
Here $p_r$ is the rollout-version student policy under the common
sampling convention; the product is one when $h=1$, and all logarithms
are natural. The current per-prefix loss satisfies
\begin{equation}
0\le F_k(s)\le C_B(1-e^{-H_k(s)})\le C_BH_k(s).
\label{eq:aopsd_loss_entropy}
\end{equation}
Consequently, the asynchronous objective discrepancy obeys
\begin{equation}
\begin{aligned}
\Delta_k
&\le C_B\,\mathbb E_{s\sim\mu_r}
\!\left[
\underbrace{d_{k,r}(s)}_{\substack{\text{prefix}\\\text{staleness}}}
\,
\underbrace{(1-e^{-H_k(s)})}_{\substack{\text{teacher}\\\text{entropy factor}}}
\right]\\
&\le C_B\,\mathbb E_{s\sim\mu_r}
\!\left[d_{k,r}(s)H_k(s)\right].
\end{aligned}
\label{eq:aopsd_main_bound}
\end{equation}
\end{aopsdtheorem}
The supporting assumptions, lemmas, and full proof are provided in
Appendix~\ref{app:aopsd_theory}.
\paragraph{Discussion.}
The bound motivates two complementary controls on prefix staleness and teacher entropy. 
The required signals are already available during training. We use the policy-version gap $k-r$ as an inexpensive proxy for prefix-distribution staleness, while teacher entropy is available from the existing teacher forward pass. Their joint use therefore requires no additional forward pass.




\subsection{Staleness Control and Entropy Weighting}
\label{sec:asyncopsd_controls}

\paragraph{Staleness-Aware Policy Control.}

Let $k$ be the current learner optimizer step and $r_i$ the step of the
published policy that generated sample $i$. Its staleness is
$s_i=\max\{0, k - r_i\}$. We reject rollouts with a staleness $s_i$ greater than $K$ using the filter:$$g_K(\tau_i) = \mathbf{1}\{s_i \le K\}$$
where $\mathbf{1}\{\cdot\}$ is the indicator function. The learner rejects staleness above $K$ when claiming a queued rollout batch, fetching another batch until an admissible one is available. Generation and learning continue asynchronously.
We set $K=6$ based on the mean staleness in \Cref{tab:asyncopsd_sync_interval} and an ablation study in \Cref{sec:asyncopsd_ablation}.
Policy version gap is an inexpensive proxy for distribution drift.

\paragraph{Token-Level Entropy Weighting.}
For an accepted completion token $h$ of rollout $i$, write $p_{i,h}$ for
the current student distribution and $q_{i,h}$ for the privileged
fixed-teacher distribution. Both use temperature-scaled logits. Define teacher entropy: 
\begin{equation}
    \begin{aligned}
    H_{i,h}&=-\sum_{v\in\mathcal V}q_{i,h}(v)\log q_{i,h}(v),\\
    u_{i,h}&=\exp(-\beta_H H_{i,h}),\qquad\beta_H\ge0. \nonumber
    \end{aligned}
\end{equation}
Over all valid accepted completion positions $\mathcal A$ in each local
learner microbatch, we normalize
\begin{equation}
    w_{i,h}
    = \frac{u_{i,h}}
    {\frac{1}{|\mathcal A|}
    \sum_{(j,m)\in\mathcal A} u_{j,m}} \nonumber
\end{equation}
These weights reduce high-entropy tokens' relative contribution and have
mean one when the numerical floor is inactive. They reuse the teacher
probabilities already computed for distillation. Setting $\beta_H=0$ recovers uniform weights. We retain per-token pointwise clipping in all conditions:
\begin{equation}
    \ell_c(q,p)=\sum_{v\in\mathcal V}
    \min\!\left\{q(v)\log\frac{q(v)}{p(v)},c\right\} \nonumber
    \label{eq:asyncopsd_clipped_loss}
\end{equation}
where $c\ge0$ is the clipping threshold. Combining staleness filtering and token-level entropy weighting with this clipped loss gives the following objective, $\mathcal C$ is the set of candidate rollouts and $T_i$ is the valid completion length of rollout $i$. For rejected rollouts, set $w_{i,h}=0$ so that all terms below are defined:

\begin{equation}
    \mathcal L_k=
    \frac{\displaystyle\sum_{i\in\mathcal C}g_K(\tau_i)
      \sum_{h=1}^{T_i}\operatorname{sg}[w_{i,h}]
      \ell_c(\operatorname{sg}[q_{i,h}],p_{i,h})}
    {\displaystyle\sum_{i\in\mathcal C}g_K(\tau_i)T_i}
    \label{eq:asyncopsd_combined_loss}
\end{equation}
where $\operatorname{sg}$ denotes stop-gradient and normalization is over
accepted tokens. The learner waits for accepted data if the denominator
would be zero. The controls are motivated by the analysis, but do not
constitute an exact correction of stale sampling. 
Entropy weighting changes the relative importance of tokens without changing the teacher target distribution, unlike the PMI target construction of Purified OPSD~\cite{shen2026purified}. High-entropy positions can also support useful reasoning branches~\cite{kaur2026rethinking}; downweighting them is a heuristic whose effect on reasoning diversity requires separate evaluation.

\section{Experiments}
\label{sec:asyncopsd_experiments}
We study the throughput--quality tradeoff of asynchronous on-policy self-distillation. 

\subsection{Experimental Setup}
\label{sec:asyncopsd_setup}

\paragraph{Models, data, and evaluation.}
Following OPSD~\cite{zhao2026self}, we consider Qwen3-1.7B, Qwen3-4B, and Qwen3-8B, using the instruct-tuned versions. For training data, we use the mathematical reasoning subset of OpenThoughts \cite{guha2026openthoughts}. We evaluate on AIME 2024, AIME 2025, and HMMT 2025 and report Avg@12. The details of the evaluation are shown  in \Cref{tab:evaluation_config}.

\paragraph{Implementation details.}
As in \cite{zhao2026self}, we use forward KL with full-vocabulary logit distillation with per-token pointwise KL clipping and the pointwise divergence clipping defined in Section~\ref{sec:asyncopsd_controls}. Following the OPSD protocol~\cite{zhao2026self}, we freeze the teacher at the initial model in all synchronous and asynchronous conditions to keep the teacher strategy matched across the empirical comparisons. The student receives the question in Thinking Mode (TM)-off, while the TM-on teacher additionally receives the reference solution. 

All runs use eight NVIDIA A100 GPUs, rank-64 LoRA with bfloat16 precision, gradient checkpointing, and FlashAttention-2 for memory efficiency.  Asynchronous execution allocates four GPUs to rollout actors and four to learners, with an effective global batch of 32. Synchronous OPSD uses all eight GPUs for alternating generation and optimization with the same effective global batch of 32. The training configuration is listed in Table~\ref{tab:trainig_config}. We repeat runs 3 times to get the mean and standard deviation to measure the throughput (samples/s) and the staleness.

\begin{table}[t]
\centering
\small
\setlength{\tabcolsep}{5pt}
\begin{tabular}{llrr}
\toprule
Model & \makecell{Sync\\Step} & Samples/s $\uparrow$ & \makecell{Mean Staleness\\$\downarrow$} \\
\midrule
Qwen3-1.7B & $1$ & $8.95 \pm 0.34$ & $4.98 \pm 0.08$ \\
  & \textbf{2} & $9.22 \pm 0.41$ & $5.35 \pm 0.28$ \\
  & $4$ & $9.41 \pm 0.49$ & $6.08 \pm 0.37$ \\
  & $8$ & $8.63 \pm 0.21$ & $7.93 \pm 0.16$ \\
\midrule
Qwen3-4B & $1$ & $5.23 \pm 0.22$ & $6.09 \pm 0.29$ \\
  & \textbf{2} & $5.53 \pm 0.46$ & $6.66 \pm 0.37$ \\
  & $4$ & $5.62 \pm 0.54$ & $7.79 \pm 0.43$ \\
  & $8$ & $5.02 \pm 0.23$ & $9.38 \pm 0.32$ \\
\midrule
Qwen3-8B & $1$ & $3.51 \pm 0.39$ & $4.44 \pm 0.28$ \\
  & \textbf{2} & $3.61 \pm 0.35$ & $5.03 \pm 0.21$ \\
  & $4$ & $3.57 \pm 0.30$ & $5.91 \pm 0.19$ \\
  & $8$ & $3.35 \pm 0.22$ & $7.88 \pm 0.29$ \\
\bottomrule
\end{tabular}
\caption{Effect of \texttt{sync\_step}. No max staleness limitation and no entropy weighting. Bold identifies selected interval.}
\label{tab:asyncopsd_sync_interval}
\end{table}

\paragraph{Efficiency metrics.}
We report accepted-training-sample throughput,
$R_{\mathrm{sample}}=\frac{N_{\mathrm{accepted}}}{T_{\mathrm{train}}}$,
where accepted samples are sequence occurrences actually consumed for optimization, and $T_{\mathrm{train}}$ is the Trainer-measured training interval. Mean staleness uses $s_i$ from accepted rollout samples.
Relative throughput is normalized to the
recorded synchronous run for the same model, following the reporting
style of asynchronous distillation studies~\cite{chen2026boldsymbol,kang2026asyncopd}.
We use sequences rather than response tokens, so our values are not numerically comparable to token-throughput results in AsyncOPD.


\subsection{Policy Export Frequency}
\label{sec:asyncopsd_sync_interval}
The \texttt{sync\_step} specifies the learner policy export frequency and indirectly changes the mean staleness on the learner side. We compare $ \texttt{sync\_step} \in\{1,2,4,8\}$ at a fixed 4/4 learner/actor allocation and
effective batch 32, with policy control and entropy weighting disabled. 

Table~\ref{tab:asyncopsd_sync_interval} shows the trade-off between throughput and staleness. Mean accepted-rollout
staleness increases monotonically with \texttt{sync\_step} for each model, whereas throughput
does not. Changing \texttt{sync\_step} from two to four increases mean throughput by only
2.0\% and 1.6\% for 1.7B and 4B and reduces it by 1.1\% for 8B; mean staleness
increases by 13.6\%, 17.0\%, and 17.5\%, respectively. 
To get a high throughput without excessively sacrificing staleness, we select $\texttt{sync\_step}=2$ as a common throughput--freshness compromise in all experiments later. 

\begin{table*}[t]
\centering
\small
\setlength{\tabcolsep}{5pt}
\begin{tabular}{llrrrr}
\toprule
Model & Method & AIME24 & AIME25 & HMMT25 & Average \\
\midrule
Qwen3-8B & OPSD \cite{zhao2026self} & \textbf{77.8} & 70.8 & 45.8 & 64.80 \\
 & Vanilla AsyncOPSD & 76.9 & 70.0 & 45.8 & 64.23 \\
 & + Staleness Control & 77.5 & 71.1 & 46.1 & 64.90 \\
 & AsyncOPSD & \textbf{77.8} & \textbf{71.7} & \textbf{47.5} & \textbf{65.67} \\
\midrule
Qwen3-4B & OPSD \cite{zhao2026self}& 76.4 & 68.3 & \textbf{46.1} & 63.60 \\
 & Vanilla AsyncOPSD & 75.0 & 67.5 & 44.4 & 62.30 \\
 & + Staleness Control & 76.1 & 68.1 & \textbf{46.1} & 63.43 \\
 & AsyncOPSD & \textbf{78.3} & \textbf{69.7} & \textbf{46.1} & \textbf{64.70} \\
\midrule
Qwen3-1.7B & OPSD \cite{zhao2026self}& 57.2 & \textbf{43.9} & 29.2 & 43.43 \\
 & Vanilla AsyncOPSD & 57.2 & 42.7 & 29.0 & 42.97 \\
 & + Staleness Control & 57.8 & 42.0 & \textbf{29.4} & 43.07 \\
 & AsyncOPSD & \textbf{58.3} & \textbf{43.9} & 29.2 & \textbf{43.80} \\
\bottomrule
\end{tabular}
\caption{We evaluate the mathematical reasoning performance of Qwen3 models using Avg@12, following the sampling settings recommended in the Qwen3 blog (temperature 1.0, maximum generation length 38k). Full benchmark details are provided in \Cref{tab:evaluation_config}.}
\label{tab:asyncopsd_accuracy}
\end{table*}

\begin{table*}[th!]
\centering
\small
\setlength{\tabcolsep}{5pt}
\begin{tabular}{llrrr}
\toprule
Model & Method & Samples/s $\uparrow$ & Relative throughput & Mean Staleness \\
\midrule
Qwen3-1.7B & Synchronous OPSD & $4.07 \pm 0.32$ & 1.00 & - \\
         & Vanilla AsyncOPSD & $9.22 \pm 0.41$ & 2.26 & $5.35 \pm 0.28$\\
         & + Staleness Control & $8.20 \pm 0.39$ & 2.01 & $4.73 \pm 0.22$ \\
         & AsyncOPSD & $8.00 \pm 0.45$ & 1.96 & $4.61 \pm 0.32$ \\
\midrule
Qwen3-4B & Synchronous OPSD & $2.40 \pm 0.33$ & 1.00 & - \\
         & Vanilla AsyncOPSD & $5.53 \pm 0.46$ & 2.31 & $6.66 \pm 0.37$\\
         & + Staleness Control & $5.46 \pm 0.35$ & 2.28 & $5.09 \pm 0.33$ \\
         & AsyncOPSD & $5.39 \pm 0.31$ & 2.25 & $5.18 \pm 0.20$ \\
\midrule
Qwen3-8B & Synchronous OPSD & $1.70 \pm 0.32$ & 1.00 & - \\
         & Vanilla AsyncOPSD & $3.61 \pm 0.35$ & 2.12 & $5.03 \pm 0.21$ \\
         & + Staleness Control & $3.43 \pm 0.39$ & 2.01 & $4.59 \pm 0.26$ \\
         & AsyncOPSD & $3.40 \pm 0.36$ & 1.99 & $4.59 \pm 0.24$ \\
\bottomrule
\end{tabular}
\caption{Accepted-training-sample throughput on eight A100s. Relative throughput is the ratio of the asynchronous and synchronous mean sample rates, computed before rounding. Synchronous rollouts have zero staleness. The comparison measures the complete implementations, including their batching and policy-refresh strategies.}
\label{tab:asyncopsd_throughput}
\end{table*}

\subsection{Vanilla Asynchronous OPSD}
\label{sec:asyncopsd_vanilla}

Vanilla AsyncOPSD decouples generation from optimization. Actors generate with refreshed student policies and place rollouts in a shared queue; learners consume them, perform model training, and periodically export the policy.
Only student LoRA adapters are published, avoiding full-model reloads on actors. 

Table~\ref{tab:asyncopsd_accuracy} reports the complete accuracy comparison, including the controlled variants discussed next. Vanilla asynchronous execution reduces the cross-benchmark average from 43.43 to 42.97 on
1.7B, from 63.60 to 62.30 on 4B, and from 64.80 to 64.23 on 8B. The decline is largest on 4B; it is not uniform across
individual benchmarks. These results motivate controlling the optimization mismatch introduced by stale rollouts. The analysis in \ref{sec:async_opsd_theory} provides a complementary motivation by identifying how prefix-distribution drift and teacher entropy interact in the objective discrepancy.

\subsection{AsyncOPSD Ablation}
\label{sec:asyncopsd_controls_ablation}

We evaluate the strategies defined in Section~\ref{sec:asyncopsd_controls}. We disable the staleness limit and entropy weighting in vanilla AsyncOPSD. The Staleness Control baseline uses $K=6$; Section~\ref{sec:asyncopsd_ablation} examines its throughput--freshness tradeoff. Full AsyncOPSD uses $K=6$ and $\beta_H=0.1$. Setting $\beta_H=0$ disables entropy weighting. The clipping threshold is $c=0.05$ for 1.7B/4B and $c=0.06$ for 8B.

\paragraph{Results comparison.}
Table~\ref{tab:asyncopsd_accuracy} compares OPSD with the asynchronous variants. We select \texttt{sync\_step}=2 to balance throughput and staleness.
Staleness control alone raises the average over vanilla AsyncOPSD by 0.10, 1.13, and 0.67 percentage points on 1.7B, 4B, and 8B.
Adding entropy weighting yields further gains of 0.73, 1.27, and 0.77 percentage points. AsyncOPSD reaches averages of 43.80, 64.70, and 65.67, exceeding OPSD by 0.37, 1.10, and 0.87 percentage points, respectively. Improvements are aggregate: for example, 1.7B HMMT25 slightly favors staleness control
alone. This ablation measures the incremental benefit of entropy weighting with staleness control fixed at $K=6$. 

\subsection{Training Throughput}
\label{sec:asyncopsd_speed}

\paragraph{Synchronous baseline.}
To get a stronger baseline, we decouple synchronous generation batch size from training microbatch size to avoid unnecessarily restricting inference throughput. All three synchronous models generate 4 prompts per GPU per vLLM call. For 1.7B and 4B, the training microbatch is also 4, with one gradient-accumulation step. For 8B, we retain a training microbatch of 2 and 2 gradient-accumulation steps, while generating both microbatches together. This reduces the number of generation calls per GPU per optimizer update for 8B model from two to one without increasing the forward/backward microbatch. The generated responses are consumed within the same optimizer update, preserving synchronous execution and an effective global batch of 32.

\paragraph{Throughput and quality tradeoff.}
Table~\ref{tab:asyncopsd_throughput} compares the updated synchronous baseline
with all three asynchronous conditions.
Synchronous OPSD achieves $4.07\pm0.32$, $2.40\pm0.33$, and $1.70\pm0.32$ samples/s on 1.7B, 4B, and 8B, respectively.
Vanilla AsyncOPSD achieves 9.22, 5.53, and 3.61 samples/s, while AsyncOPSD achieves 8.00, 5.39, and 3.40 samples/s. Relative to the updated synchronous means, vanilla execution improves throughput by factors of 2.26, 2.31, and 2.12; staleness control alone achieves 2.01, 2.28, and 2.01 times the synchronous throughput. AsyncOPSD achieves \textbf{1.96$\times$, 2.25$\times$, and 1.99$\times$ throughput} with reported average accuracy gains over OPSD of 0.37, 1.10, and 0.87 percentage points on 1.7B, 4B, and 8B, respectively. It retains 86.8\%, 97.5\%, and 94.2\% of vanilla throughput. Relative to staleness control alone, entropy weighting changes mean throughput by $-2.4$\%, $-1.3$\%, and $-0.9$\%.



\paragraph{Sources of throughput improvement.}
The asynchronous implementation combines several systems changes.
1) Larger generation batches: Each rollout actor generates 32 completions per vLLM call, compared with four completions per GPU in the synchronous baseline. Larger batches can improve inference efficiency through higher GPU utilization.
2) Lower policy-refresh overhead. The synchronous implementation refreshes vLLM by merging the LoRA adapters every step. The asynchronous implementation instead exports only the updated LoRA adapters at a configurable $\texttt{sync\_step}$.
3) Nonlinear learner scaling. Doubling the number of learner GPUs does not necessarily halve optimization time. Communication, synchronization, and small per-device workloads can limit scaling efficiency. Consequently, splitting GPUs into a learner pool and rollout actors pool can improve overall utilization despite using fewer GPUs for optimization.
4) Overlapping generation and optimization. Rollout actors generate upcoming batches while learners optimize previously generated samples. This overlap reduces idle periods associated with alternating generation and optimization in synchronous execution.

\begin{figure}[t]
\centering
\begin{subfigure}[t]{0.36\textwidth}
\centering
\includegraphics[width=\linewidth]{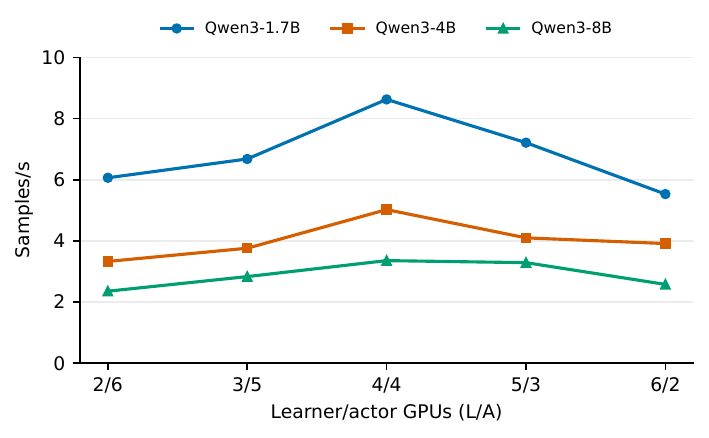}
\caption{Accepted-training-sample throughput.}
\label{fig:asyncopsd_gpu_allocation_throughput}
\end{subfigure}
\hfill
\begin{subfigure}[t]{0.36\textwidth}
\centering
\includegraphics[width=\linewidth]{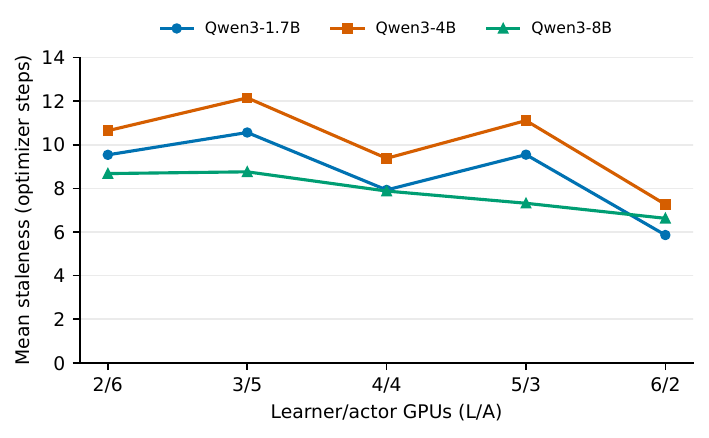}
\caption{Mean accepted-rollout staleness.}
\label{fig:asyncopsd_gpu_allocation_staleness}
\end{subfigure}
\caption{GPU-allocation ablation on eight A100s. L/A denotes learner/actor GPUs. 
(a) shows mean accepted-training-sample throughput; panel (b) averages mean rollout staleness in learner optimizer steps. The sweep uses $\texttt{sync\_step}=8$ vanilla AsyncOPSD.
}
\vspace{-0.2in}
\label{fig:asyncopsd_gpu_allocation}
\end{figure}

\vspace{-0.1in}
\subsection{Ablation Study}
\label{sec:asyncopsd_ablation}

\paragraph{GPU-Allocation Ablation.}
\Cref{fig:asyncopsd_gpu_allocation} shows the learner/actor allocation over 2/6, 3/5, 4/4, 5/3, and 6/2
on eight A100s with vanilla asynchronous execution.
The balanced 4/4 allocation provides the highest mean sample throughput for each model: 8.63, 5.02, and 3.35 samples/s for 1.7B, 4B, and 8B. 
Allocating more GPUs to either generation or learning alone is insufficient: the two stages must supply and consume samples at compatible rates. We adopt 4/4 as a common operating point.


\paragraph{Staleness-Cutoff Ablation.}
We sweep $K\in\{4,6,8,10\}$ on Qwen3-1.7B with $\beta_H=0$. Throughput relative to the sweep's $K=6$ mean is 0.755, 1.000, 1.114, and 1.108, respectively. Increasing $K$ from 4 to 6 improves throughput by 32.4\% and reduces rejection among claimed rollouts from 33.9\% to 10.7\%. Compared with $K=8$, $K=6$ retains 89.8\% of throughput while reducing mean accepted-rollout staleness from 6.12 to 5.35; $K=10$ provides no further mean throughput gain. For the tested values K = 8 and K = 10, no rejections occurred. It is similar to the vanilla AsyncOPSD.  So we choose $K=6$ as a conservative throughput--freshness compromise. Appendix~\ref{app:asyncopsd_staleness_ablation} reports the details. 

\paragraph{Entropy-Weighting Ablation.}
On Qwen3-1.7B, we select $\beta_H=0.1$ based on evaluation results among the tested coefficients $\{0.03,0.1,0.3,1.0\}$. Full results are provided in Appendix~\ref{app:asyncopsd_entropy_ablation}.

\vspace{-0.1in}
\section*{Conclusion}
We introduced AsyncOPSD, which improves training efficiency by decoupling rollout generation from learner updates. Our analysis bounds the resulting objective discrepancy in terms of prefix staleness and teacher entropy, motivating staleness control and token-level entropy weighting. Across three model sizes and mathematical reasoning benchmarks, AsyncOPSD achieves $1.96$--$2.25\times$ the sample throughput of an optimized synchronous baseline while increasing average benchmark accuracy over OPSD by $0.37$--$1.10$ percentage points. 
\section*{Limitations}
Our experiments use eight NVIDIA A100 GPUs and models ranging from 1.7B to 8B parameters. Within this budget, comparisons across three model sizes and targeted ablations provide evidence for the throughput benefits of asynchronous execution and the effectiveness of staleness control and entropy weighting in the evaluated settings. A broader exploration of model scales and training configurations is left to future work.




\bibliography{custom}

\newpage

\newpage
\appendix

\section{Theoretical Analysis}
\label{sec:appendix}
\label{app:aopsd_theory}

This appendix supplies the sampling assumptions and intermediate
lemmas used in Theorem~\ref{thm:aopsd_discrepancy}, followed by its full
proof. The learner update $k$ and rollout version $r\le k$ are fixed.
The first three subsections evaluate both the student and privileged
teacher at the current parameters $\theta_k$ when computing the loss.
Appendix~\ref{app:aopsd_frozen_teacher} extends the result to a teacher
frozen at initialization, as used in our experiments.

\subsection{Assumptions and prefix distributions}

\begin{aopsdassumption}[Common sampling protocol and coverage]
\label{ass:aopsd_sampling}
The rollout and current prefix laws use the same problem--reference
distribution $\nu(x,a)$, the same fixed token-position probabilities
$\alpha_h$ for $h\in\{1,\ldots,T\}$, and the same continuation
convention after termination. Prefixes are generated autoregressively
using the corresponding student policies $p_r$ and $p_k$, and
$\mu_k\ll\mu_r$.
\end{aopsdassumption}

The notation $\mu_k\ll\mu_r$ means absolute continuity. In the
discrete prefix setting, it is the support condition
\[
\mu_r(s)=0\quad\Longrightarrow\quad\mu_k(s)=0.
\]
Equivalently, every prefix with positive current-policy probability
must also have positive rollout-policy probability. This condition
allows the current prefix law to be represented by reweighting the
rollout law; it does not compare the magnitudes of the probabilities.
It concerns prefixes the rollout policy could generate, rather than
requiring every such prefix to appear in a finite realized buffer.

Under Assumption~\ref{ass:aopsd_sampling}, for $j\in\{r,k\}$,
\begin{equation}
\mu_j(s_h)
=\nu(x,a)\,\alpha_h\prod_{u<h}p_j(y_u\mid s_u).
\label{eq:aopsd_prefix_law}
\end{equation}
Full-support softmax sampling on a common vocabulary is sufficient
for the coverage condition. The common problem and token-position
factors cancel when comparing the two prefix laws.

Let $z_k^S(s)$ and $z_k^T(s)$ be the student and privileged teacher
logits, respectively. They satisfy
\begin{equation}
\begin{aligned}
p_k&=\operatorname{softmax}(z_k^S),\\
q_k&=\operatorname{softmax}(z_k^T),\\
b_k(s)&=z_k^T(s)-z_k^S(s).
\end{aligned}
\label{eq:aopsd_context_contrast}
\end{equation}

\begin{aopsdassumption}[Bounded privileged logit contrast]
\label{ass:aopsd_contrast}
There exists a finite constant $B\ge0$ such that, at the analyzed
update and uniformly over the support of $\mu_r$,
\begin{equation}
\operatorname{osc}(b_k(s))
:=\max_v b_{k,v}(s)-\min_v b_{k,v}(s)\le B.
\label{eq:aopsd_contrast_assumption}
\end{equation}
\end{aopsdassumption}

The finite-support probability lower bound in Assumption~3.1 of
$f$-OPD~\cite{chen2026boldsymbol} is sufficient for this contrast
condition when it holds on the same support used to evaluate the
loss. Indeed, if $p_k(v\mid s),q_k(v\mid s)\ge\epsilon>0$, then
\[
-\log(1/\epsilon)
\le\log\frac{q_k(v\mid s)}{p_k(v\mid s)}
\le\log(1/\epsilon).
\]
The vector $b_k(s)$ differs from $\log(q_k/p_k)$ only by a common
scalar, so its oscillation is at most $2\log(1/\epsilon)$.
A lower bound imposed only on truncated distributions does not
automatically imply the full-vocabulary condition. We state the
contrast assumption directly because parameter sharing alone does
not guarantee a small $B$, and $C_B=B(e^B-1)/2$ can be large when
the admissible contrast is large.

\subsection{Supporting lemmas}

\begin{aopsdlemma}[Change of measure and student staleness]
\label{lem:aopsd_staleness}
Under Assumption~\ref{ass:aopsd_sampling}, define
\begin{equation}
\rho_{k,r}(s)=\frac{\mu_k(s)}{\mu_r(s)},
\qquad d_{k,r}(s)=|1-\rho_{k,r}(s)|.
\label{eq:aopsd_staleness}
\end{equation}
The prefix likelihood ratio satisfies
\begin{equation}
\rho_{k,r}(s_h)
=\prod_{u<h}\frac{p_k(y_u\mid s_u)}{p_r(y_u\mid s_u)}.
\label{eq:aopsd_prefix_ratio}
\end{equation}
Whenever the two objectives are finite,
\begin{equation}
\begin{aligned}
\Delta_k
&=\left|\mathbb E_{\mu_r}
\left[(1-\rho_{k,r}(s))F_k(s)\right]\right|\\
&\le\mathbb E_{\mu_r}
\left[d_{k,r}(s)F_k(s)\right].
\end{aligned}
\label{eq:aopsd_staleness_loss}
\end{equation}
Moreover, $\mathbb E_{\mu_r}[d_{k,r}(s)]\le2$.
\end{aopsdlemma}

\begin{proof}
Dividing the factorizations in Eq.~\eqref{eq:aopsd_prefix_law}
cancels $\nu(x,a)$ and $\alpha_h$, which gives
Eq.~\eqref{eq:aopsd_prefix_ratio}. Absolute continuity ensures that
the likelihood ratio exists $\mu_r$-almost everywhere and permits
the change of measure
\[
J_{\mathrm{sync}}^k
=\mathbb E_{\mu_r}[\rho_{k,r}(s)F_k(s)].
\]
Both objectives use exactly the same current loss $F_k(s)$;
only the distribution over prefixes changes. Subtracting the
last expression from $J_{\mathrm{async}}^k$ therefore yields
the equality in Eq.~\eqref{eq:aopsd_staleness_loss}.
The triangle inequality and $F_k(s)\ge0$ give its upper bound.
Finally, since $\mathbb E_{\mu_r}[\rho_{k,r}]=1$,
\[
\mathbb E_{\mu_r}[d_{k,r}]
\le\mathbb E_{\mu_r}[1+\rho_{k,r}]=2.
\]
Equivalently, this expectation is
$2\operatorname{TV}(\mu_r,\mu_k)$ under the convention
$\operatorname{TV}(P,Q)=\tfrac12\|P-Q\|_1$.
\end{proof}

\begin{aopsdlemma}[Entropy-sensitive self-distillation loss]
\label{lem:aopsd_entropy}
Under Assumption~\ref{ass:aopsd_contrast}, the current forward-KL
loss satisfies, for every prefix in the support of $\mu_r$,
\[
0\le F_k(s)
\le C_B(1-e^{-H_k(s)})
\le C_BH_k(s),
\]
where $H_k(s)$ and $C_B$ are defined in
Eq.~\eqref{eq:aopsd_entropy}.
\end{aopsdlemma}

\begin{proof}
Suppress $k,s$ and write $p=\operatorname{softmax}(z^S)$ and
$q=\operatorname{softmax}(z^T)$. Because $z^S=z^T-b$, set
\[
\delta_v=-b_v+\max_u b_u.
\]
Assumption~\ref{ass:aopsd_contrast} gives $0\le\delta_v\le B$.
Adding a common scalar to logits does not change their softmax,
so
\begin{equation}
p_v=\frac{q_v e^{\delta_v}}{Z},
\qquad Z=\sum_uq_u e^{\delta_u}\ge1.
\label{eq:aopsd_exponential_tilt}
\end{equation}

First, bound the distance between the two predictive distributions:
\begin{align}
\operatorname{TV}(p,q)
&=\frac{1}{2Z}\sum_vq_v|e^{\delta_v}-Z|
\nonumber\\
&\le\frac{1}{2Z}\sum_{v,u}q_vq_u
|e^{\delta_v}-e^{\delta_u}|
\nonumber\\
&\le\frac{e^B-1}{2}
\left(1-\sum_vq_v^2\right).
\label{eq:aopsd_tv_tilt}
\end{align}
The first inequality substitutes $Z=\sum_uq_u e^{\delta_u}$
and applies the triangle inequality. The second uses $Z\ge1$,
$|e^{\delta_v}-e^{\delta_u}|\le e^B-1$, and the fact that the
summand is zero for $u=v$.

Next, Eq.~\eqref{eq:aopsd_exponential_tilt} implies
\[
\log\frac{p_v}{q_v}=\delta_v-\log Z.
\]
Consequently,
\begin{align}
\operatorname{KL}(q\|p)
&=\log Z-\sum_vq_v\delta_v,
\nonumber\\
\operatorname{KL}(p\|q)
&=\sum_vp_v\delta_v-\log Z.
\label{eq:aopsd_kl_tilt}
\end{align}
Adding the two identities cancels the normalization term. With
$I_+=\{v:p_v>q_v\}$, the remaining expression obeys
\begin{align}
&\operatorname{KL}(q\|p)+\operatorname{KL}(p\|q)
\nonumber\\
&\qquad=\sum_v(p_v-q_v)\delta_v
\nonumber\\
&\qquad\le B\sum_{v\in I_+}(p_v-q_v)
=B\operatorname{TV}(p,q).
\label{eq:aopsd_symmetric_kl}
\end{align}
Here the components outside $I_+$ contribute nonpositive terms,
and $\delta_v\le B$ on $I_+$. Since the reverse KL is nonnegative,
combining this inequality with Eq.~\eqref{eq:aopsd_tv_tilt} gives
\begin{equation}
F_k(s)\le
\frac{B(e^B-1)}{2}
\left(1-\sum_vq_v^2\right).
\label{eq:aopsd_gini_loss_bound}
\end{equation}

Finally, Jensen's inequality for the convex function $-\log$
yields
\[
H(q)=\sum_vq_v[-\log q_v]
\ge-\log\!\left(\sum_vq_v^2\right).
\]
Thus $e^{-H(q)}\le\sum_vq_v^2$, and
\[
1-\sum_vq_v^2\le1-e^{-H(q)}\le H(q).
\]
Substituting this relation into
Eq.~\eqref{eq:aopsd_gini_loss_bound}, together with the
nonnegativity of KL, proves the lemma.
\end{proof}

\subsection{Proof of Theorem~\ref{thm:aopsd_discrepancy}}

\begin{proof}[Proof of Theorem~\ref{thm:aopsd_discrepancy}]
By Assumptions~\ref{ass:aopsd_sampling} and~\ref{ass:aopsd_contrast},
the common prefix protocol, coverage, and uniform logit contrast
condition hold. Lemma~\ref{lem:aopsd_entropy} therefore establishes
the per-prefix
loss bound in Eq.~\eqref{eq:aopsd_loss_entropy}.
In particular, $0\le F_k(s)\le C_B$ on the rollout support and,
by coverage, on the current-policy support. Hence both population
objectives are finite. Lemma~\ref{lem:aopsd_staleness} also gives
$\mathbb E_{\mu_r}[d_{k,r}]\le2$, ensuring that the product with
the bounded loss envelope is integrable.

Apply the change-of-measure lemma and then the entropy-sensitive
loss bound pointwise:
\begin{align}
\Delta_k
&=\left|\mathbb E_{\mu_r}
\left[(1-\rho_{k,r}(s))F_k(s)\right]\right|
\nonumber\\
&\le\mathbb E_{\mu_r}
\left[d_{k,r}(s)F_k(s)\right]
\nonumber\\
&\le C_B\,\mathbb E_{\mu_r}
\left[d_{k,r}(s)(1-e^{-H_k(s)})\right]
\nonumber\\
&\le C_B\,\mathbb E_{\mu_r}
\left[d_{k,r}(s)H_k(s)\right].
\label{eq:aopsd_appendix_proof}
\end{align}
The last step uses $1-e^{-H}\le H$ for $H\ge0$ and
$d_{k,r}(s)\ge0$. This proves Eq.~\eqref{eq:aopsd_main_bound}.
The product remains inside the expectation throughout, so
independence of staleness and entropy is not required.
\end{proof}

\paragraph{Sampling probabilities and loss normalization.}
The likelihood ratio must describe the actual distributions that
generate prefixes. If generation uses a different temperature or
other probability transformation, the prefix laws and ratio must use
those sampling probabilities; the loss envelope uses the $p_k,q_k$
appearing in the KL. Hard top-$p$ or top-$k$ truncation can violate
coverage and requires separate support handling. The simple product
formula also assumes the common fixed token-position distribution
in Assumption~\ref{ass:aopsd_sampling}; filtering or response-length
normalization may alter the effective training-prefix laws.

\paragraph{Meaning of the two factors.}
The staleness signal accumulates probability mismatch over the previous tokens of a prefix. It is not interchangeable with policy update lag or a local one-token rollout KL without an additional argument accounting for earlier positions. Teacher entropy controls the current loss envelope under the contrast assumption; it does not represent a separate cached-target mismatch. Because both learner-side predictions use $\theta_k$ and the privileged reference is fixed, the comparison changes only the prefix law. The theorem therefore concerns the generating population laws, rather than the empirical distribution of a finite realized buffer.

\subsection{Extension to a frozen privileged teacher}
\label{app:aopsd_frozen_teacher}

We now fix the privileged teacher at the initial parameters
$\theta_0$, matching the teacher strategy in
Section~\ref{sec:asyncopsd_setup}. For $s_h=(x,a,h,y_{<h})$, define
\begin{equation}
\begin{aligned}
q_0(\cdot\mid s_h)
&=\pi_{\theta_0}(\cdot\mid x,a,y_{<h}),\\
F_k^{\mathrm{fr}}(s)
&=\operatorname{KL}\!\left(q_0(\cdot\mid s)\,\middle\|\,
p_k(\cdot\mid s)\right),\\
H_0(s)&=-\sum_{v\in\mathcal V}q_0(v\mid s)\log q_0(v\mid s).
\end{aligned}
\label{eq:aopsd_frozen_loss}
\end{equation}
The teacher and its entropy are fixed functions of the prefix; they
can still vary across prefixes. The student and the prefix laws
$\mu_r,\mu_k$ remain as defined above. Set
\begin{equation}
\begin{aligned}
J_{\mathrm{async,fr}}^k&=\mathbb E_{\mu_r}[F_k^{\mathrm{fr}}(s)],\\
J_{\mathrm{sync,fr}}^k&=\mathbb E_{\mu_k}[F_k^{\mathrm{fr}}(s)],\\
\Delta_k^{\mathrm{fr}}
&=\left|J_{\mathrm{async,fr}}^k-J_{\mathrm{sync,fr}}^k\right|.
\end{aligned}
\label{eq:aopsd_frozen_objectives}
\end{equation}
Both objectives use the same frozen teacher and current student;
only their prefix distributions differ.

\begin{aopsdassumption}[Bounded frozen-teacher logit contrast]
\label{ass:aopsd_frozen_contrast}
Let $q_0=\operatorname{softmax}(z_0^T)$ and
$p_k=\operatorname{softmax}(z_k^S)$. At the analyzed update, there
exists a finite $B_k^{\mathrm{fr}}\ge0$ such that, uniformly over
the support of $\mu_r$,
\begin{equation}
\operatorname{osc}\!\left(z_0^T(s)-z_k^S(s)\right)
\le B_k^{\mathrm{fr}}.
\label{eq:aopsd_frozen_contrast}
\end{equation}
The logits are those used in the loss, including any temperature scaling.
\end{aopsdassumption}

Here the contrast is the teacher--student logit difference on the same prefix. Bounding its oscillation is equivalent to bounding the difference in their log odds for every token pair $u,v$:
\[
\left|\log\frac{q_0(v\mid s)}{q_0(u\mid s)}
-\log\frac{p_k(v\mid s)}{p_k(u\mid s)}\right|
\le B_k^{\mathrm{fr}}.
\]
This is a regularity assumption on the predictive distributions, not an additional training control.

\begin{aopsdtheorem}[Frozen-teacher staleness--entropy bound]
\label{thm:aopsd_frozen_discrepancy}
Under Assumptions~\ref{ass:aopsd_sampling}
and~\ref{ass:aopsd_frozen_contrast}, let $d_{k,r}$ be the student
prefix staleness in Eq.~\eqref{eq:aopsd_staleness}, and write
$C_B=B(e^B-1)/2$. Then
\begin{equation}
\begin{aligned}
0\le F_k^{\mathrm{fr}}(s)
&\le C_{B_k^{\mathrm{fr}}}(1-e^{-H_0(s)})\\
&\le C_{B_k^{\mathrm{fr}}}H_0(s),
\end{aligned}
\label{eq:aopsd_frozen_entropy}
\end{equation}
and
\begin{equation}
\begin{aligned}
\Delta_k^{\mathrm{fr}}
&\le C_{B_k^{\mathrm{fr}}}\,
\mathbb E_{\mu_r}\!\left[d_{k,r}(s)(1-e^{-H_0(s)})\right]\\
&\le C_{B_k^{\mathrm{fr}}}\,
\mathbb E_{\mu_r}\!\left[d_{k,r}(s)H_0(s)\right].
\end{aligned}
\label{eq:aopsd_frozen_bound}
\end{equation}
\end{aopsdtheorem}

\begin{proof}
The proof of Lemma~\ref{lem:aopsd_entropy} uses only the softmax
representations and the bounded oscillation of their logit difference.
For a fixed prefix, set $b=z_0^T-z_k^S$ and
$\delta_v=-b_v+\max_u b_u$. Then
$0\le\delta_v\le B_k^{\mathrm{fr}}$ and
\[
p_k(v\mid s)
=\frac{q_0(v\mid s)e^{\delta_v}}
{\sum_u q_0(u\mid s)e^{\delta_u}}.
\]
Thus the same total-variation and symmetric-KL inequalities give
\begin{align*}
F_k^{\mathrm{fr}}(s)
&\le C_{B_k^{\mathrm{fr}}}
\left(1-\sum_v q_0(v\mid s)^2\right)\\
&\le C_{B_k^{\mathrm{fr}}}(1-e^{-H_0(s)}).
\end{align*}
Nonnegativity of KL and $1-e^{-H_0}\le H_0$ prove
Eq.~\eqref{eq:aopsd_frozen_entropy}. In particular, the loss is
bounded on both prefix supports. By the same change of measure as
in Lemma~\ref{lem:aopsd_staleness},
\begin{align*}
\Delta_k^{\mathrm{fr}}
&=\left|\mathbb E_{\mu_r}
\left[(1-\rho_{k,r}(s))F_k^{\mathrm{fr}}(s)\right]\right|\\
&\le\mathbb E_{\mu_r}\!\left[d_{k,r}(s)F_k^{\mathrm{fr}}(s)\right].
\end{align*}
Substitution gives Eq.~\eqref{eq:aopsd_frozen_bound}.
\end{proof}

\paragraph{Accounting for student drift.}
The frozen-teacher contrast assumption must be checked separately from Assumption~\ref{ass:aopsd_contrast}. An optional sufficient condition separates the initial teacher--student difference from subsequent student drift. Let $z_0^S$ denote the initial student's logits. If, uniformly over the analyzed rollout support,
\begin{equation}
\begin{aligned}
\operatorname{osc}(z_0^T(s)-z_0^S(s))&\le B_{\mathrm{init}},\\
\operatorname{osc}(z_k^S(s)-z_0^S(s))&\le D_k,
\end{aligned}
\label{eq:aopsd_frozen_drift}
\end{equation}
then subadditivity of oscillation yields
\begin{align*}
\operatorname{osc}(z_0^T-z_k^S)
&\le\operatorname{osc}(z_0^T-z_0^S)\\
&\quad+\operatorname{osc}(z_k^S-z_0^S)\\
&\le B_{\mathrm{init}}+D_k.
\end{align*}
Hence $B_{\mathrm{init}}+D_k$ is one admissible choice of
$B_k^{\mathrm{fr}}$; the theorem only requires
Assumption~\ref{ass:aopsd_frozen_contrast}, not this decomposition.
Freezing the teacher does not itself bound accumulated student drift or guarantee the same numerical constant as in the evolving-teacher
case. A bound on rollout age $k-r$ also does not by itself bound drift from initialization.

\paragraph{Implications and scope.}
The same interaction between prefix staleness and teacher entropy therefore holds for a frozen teacher under the stated contrast bound. In particular, $\mu_r=\mu_k$ implies $\Delta_k^{\mathrm{fr}}=0$, even if the current student differs substantially from the teacher.
There is no teacher-version mismatch on a fixed prefix: cached and recomputed frozen-teacher predictions coincide when the privileged context and evaluation settings are unchanged. Low teacher entropy alone cannot ensure a small distillation loss without a uniform bound
on the teacher--student logit contrast: for
$q_\epsilon=(1-\epsilon,\epsilon)$ and $p=(1/2,1/2)$,
$H(q_\epsilon)\to0$ but
$\operatorname{KL}(q_\epsilon\|p)\to\log 2$ as $\epsilon\to0$; the logit contrast then diverges.

\section{Experiment Details}

\begin{table}[h]
\centering
\small
\setlength{\tabcolsep}{5pt}
\begin{tabular}{ll}
\toprule
Parameter & Training \\
\midrule
Optimizer & AdamW \\
Learning Rate & $5 \times 10^{-6}$   \\
Learning Rate Scheduler & Constant \\
Effective Batch Size & 32 \\
Dataset shuffling seed & 42 \\
\midrule
LoRA Rank  $r$ &  64 \\
LoRA Alpha $\alpha$ &  128 \\
LoRA Target Modules & \makecell[l]{q\_proj, k\_proj, v\_proj,\\o\_proj, gate\_proj,\\up\_proj, down\_proj} \\
\midrule
Max Completion Length & 1024 \\
\makecell[l]{Number of Generations\\per Prompt} & 1 \\
Sampling Temperature & 1.1 \\ 
Training Steps & 100 \\
\bottomrule
\end{tabular}
\caption{Training Configuration.}
\label{tab:trainig_config}
\end{table}

\begin{table}[h]
\centering
\small
\setlength{\tabcolsep}{5pt}
\begin{tabular}{ll}
\toprule
Parameter & Evaluation \\
\midrule
Max New Tokens & 38912 \\
Thinking Mode & Enabled \\
Top-p & 0.95 \\
Top-k & -1 \\
Min-p & 0.0 \\
Samples per Prompt & 12 \\
Temperature & 1.0 \\
\bottomrule
\end{tabular}
\caption{Evaluation Configuration.}
\label{tab:evaluation_config}
\end{table}

\begin{table}[h]
\centering
\small
\setlength{\tabcolsep}{5pt}
\begin{tabular}{llrr}
\toprule
Model & L/A & Samples/s $\uparrow$ & Mean staleness $\downarrow$  \\
\midrule
Qwen3-1.7B & 2/6 & $6.06 \pm 0.29$ & $9.54 \pm 0.43$ \\
Qwen3-1.7B & 3/5 & $6.68 \pm 0.25$ & $10.56 \pm 0.19$ \\
Qwen3-1.7B & 4/4 & $8.63 \pm 0.21$ & $7.93 \pm 0.16$ \\
Qwen3-1.7B & 5/3 & $7.21 \pm 0.15$ & $9.55 \pm 0.11$ \\
Qwen3-1.7B & 6/2 & $5.53 \pm 0.13$ & $5.86 \pm 0.05$ \\
\midrule
Qwen3-4B & 2/6 & $3.32 \pm 0.24$ & $10.65 \pm 0.43$ \\
Qwen3-4B & 3/5 & $3.76 \pm 0.28$ & $12.15 \pm 0.67$ \\
Qwen3-4B & 4/4 & $5.02 \pm 0.23$ & $9.38 \pm 0.32$ \\
Qwen3-4B & 5/3 & $4.10 \pm 0.29$ & $11.11 \pm 0.41$ \\
Qwen3-4B & 6/2 & $3.91 \pm 0.17$ & $7.27 \pm 0.18$ \\
\midrule
Qwen3-8B & 2/6 & $2.35 \pm 0.40$ & $8.68 \pm 0.71$ \\
Qwen3-8B & 3/5 & $2.83 \pm 0.25$ & $8.76 \pm 0.40$ \\
Qwen3-8B & 4/4 & $3.35 \pm 0.22$ & $7.88 \pm 0.29$ \\
Qwen3-8B & 5/3 & $3.28 \pm 0.25$ & $7.32 \pm 0.41$ \\
Qwen3-8B & 6/2 & $2.57 \pm 0.27$ & $6.63 \pm 0.43$ \\
\bottomrule
\end{tabular}
\caption{GPU-allocation ablation on eight A100s at sync interval 8, No max staleness limitation, and uniform token weights. L/A denotes learner/actor GPUs.}
\label{tab:asyncopsd_gpu_allocation}
\end{table}

\subsection{Staleness-Cutoff Ablation}
\label{app:asyncopsd_staleness_ablation}

We evaluate maximum rollout age $K\in\{4,6,8,10\}$ on Qwen3-1.7B using eight A100s split into four learner and four actor GPUs, with policy export interval $\texttt{sync\_step}=2$, clipping threshold $c=0.05$, and uniform token weights ($\beta_H=0$). This configuration is the staleness-control-only variant; vanilla AsyncOPSD has no active age cutoff.

Table~\ref{tab:asyncopsd_staleness_cutoff} reports the directly measured results of this additional sweep. We compute relative throughput as $R_K=\bar v_K/\bar v_6$, where $\bar v_K$ is the mean accepted-sample throughput across the three executions at $K$. The denominator is the measured $K=6$ mean from this sweep; the absolute measurements in Table~\ref{tab:asyncopsd_throughput} come from separate executions and are retained. 
The rejection percentage covers claimed microbatches and excludes generated samples left in the queue at shutdown.

Tightening the cutoff to $K=4$ lowers mean accepted staleness to 3.79 and rejects about one third of claimed rollouts. At $K=6$, mean throughput increases to 9.22 samples/s and rejection drops to 10.70\%. Relaxing the cutoff to $K=8$ yields a further 11.4\% throughput gain, accompanied by older accepted rollouts. No rejections occur at $K=8$ or $K=10$ in these executions, and their mean throughputs are similar. We retain $K=6$ to enforce a tighter age bound while preserving most of the observed peak throughput.

\begin{table}[h]
\centering
\small
\setlength{\tabcolsep}{5pt}
\begin{tabular}{rrrrrr}
\toprule
$K$ & Samples/s $\uparrow$ & $R_K$ & Mean Staleness $\downarrow$ & Rejected (\%) $\downarrow$ \\
\midrule
4 & $6.96 \pm 0.36$ & 0.755 & $3.79 \pm 0.21$ & $33.91 \pm 0.91$ \\
6 & $9.22 \pm 0.41$ & 1.000 & $5.35  \pm 0.28$ & $10.70 \pm 1.39$ \\
8 & $10.27 \pm 0.33$& 1.114 & $6.12 \pm 0.38$ & $0.00 \pm 0.00$ \\
10 & $10.22 \pm 0.31$ & 1.108 & $6.20 \pm 0.39$ & $0.00 \pm 0.00$ \\
\bottomrule
\end{tabular}
\caption{Staleness-cutoff ablation on Qwen3-1.7B with $\beta_H=0$. Relative throughput ($R_K=\bar v_K/\bar v_6$) is normalized to this sweep's measured $K=6$ mean before rounding. Staleness is measured in learner optimizer steps; rejection is the fraction of claimed queue microbatches discarded by the cutoff.}
\label{tab:asyncopsd_staleness_cutoff}
\end{table}

\begin{table}[b]
\centering
\small
\setlength{\tabcolsep}{10pt}
\begin{tabular}{rr}
\toprule
$\beta_H$ & AIME 2025 \\
\midrule
0.03 & 42.50 \\
\textbf{0.1} & \textbf{43.90} \\
0.3 & 42.78 \\
1.0 & 40.28 \\
\bottomrule
\end{tabular}
\caption{Effect of the entropy-weighting coefficient on Qwen3-1.7B. Bold marks the highest AIME 2025 Avg@12 among the tested values.}
\label{tab:asyncopsd_entropy_beta}
\end{table}

\subsection{Entropy-Weighting Ablation}
\label{app:asyncopsd_entropy_ablation}

We evaluate the entropy-weighting coefficient $\beta_H$ on Qwen3-1.7B using AIME 2025. Table~\ref{tab:asyncopsd_entropy_beta} reports Avg@12 for four coefficients. The selected value $\beta_H=0.1$ achieves the highest score, 43.90, while both the smaller coefficient 0.03 and the larger coefficients 0.3 and 1.0 yield lower scores. These results support $\beta_H=0.1$ within the tested model and benchmark.

\end{document}